\documentclass[11pt, a4paper]{article}

\usepackage[margin=1.19in]{geometry}
\usepackage[utf8]{inputenc} 
\usepackage[T1]{fontenc}    
\usepackage{hyperref}       
\usepackage{url}            
\usepackage{booktabs}       
\usepackage{nicefrac}       
\usepackage{microtype}      
\usepackage{algorithm}
\usepackage[noend]{algorithmic} 
\usepackage{times}
\usepackage{latexsym}
\usepackage{graphicx}
\usepackage{subcaption}
\usepackage[english]{babel} 
\usepackage{amsmath,amsfonts,amsthm} 
\usepackage{enumitem}
\usepackage{wrapfig}
\usepackage{amssymb}
\usepackage{thmtools}
\usepackage{natbib}
\usepackage{nameref,cleveref,}
\usepackage{multirow}
\usepackage{pbox}
\usepackage{pgfplots}
\usepackage{fullpage}
\usepackage{subcaption}
\usepackage{authblk}
\usepackage{pifont}
\usepackage{xcolor}
\usepackage{mathtools}

\newtheorem{theorem}{Theorem}

\newtheorem{example}{Example}
\declaretheorem[name=Corollary, sibling=theorem, refname={corollary,corollaries}, Refname={Corollary,Corollaries}]{corollary}
\declaretheorem[name=Lemma, sibling=theorem, refname={lemma,lemmas}, Refname={Lemma,Lemmas}]{lemma}

\declaretheorem[name=Definition, sibling=theorem, refname={definition,definitions}, Refname={Definition,Definitions}]{definition}

\declaretheorem[name=Assumption, refname={assumption,assumptions}, Refname={Assumption,Assumptions}]{assumption}

\newcommand{\cmark}{\textcolor{green}{\ding{51}}}
\newcommand{\xmark}{\textcolor{red}{\ding{55}}}

\title{Combinatorial Bandits without Total Order for Arms}
\author[1]{Shuo Yang$\thanks{Email: \url{yangshuo_ut@utexas.edu}}$}
\author[1]{Tongzheng Ren}
\author[1, 2]{Inderjit S. Dhillon}
\author[1]{Sujay Sanghavi}
\affil[1]{The University of Texas at Austin}
\affil[2]{Amazon}

\input{definition}

\begin{document}
\maketitle
\begin{abstract}
    We consider the combinatorial bandits problem, where at each time step, the online learner selects a size-$k$ subset $s$ from the arms set $\mathcal{A}$, where $\abs{\Acal} = n$, and observes a stochastic reward of each arm in the selected set $s$. The goal of the online learner is to minimize the regret, induced by not selecting $s^*$ which maximizes the expected total reward. Specifically, we focus on a challenging setting where 1) the reward distribution of an arm depends on the set $s$ it is part of, and crucially 2) there is \textit{no total order} for the arms in $\mathcal{A}$. 
    
    In this paper, we formally present a reward model that captures set-dependent reward distribution and assumes no total order for arms. Correspondingly, we propose an Upper Confidence Bound (UCB) algorithm that maintains UCB for each individual arm and selects the arms with top-$k$ UCB. We develop a novel regret analysis and show an $O\rbr{\frac{k^2 n \log T}{\epsilon}}$ gap-dependent regret bound as well as an $O\rbr{k^2\sqrt{n T \log T}}$ gap-independent regret bound. We also provide a lower bound for the proposed reward model, which shows our proposed algorithm is near-optimal for any constant $k$. Empirical results on various reward models demonstrate the broad applicability of our algorithm.
\end{abstract}
\section{Introduction}\label{sec:introduction}

Arising from various real-world applications (online advertisement, recommendation systems, etc.), combinatorial bandits \citep{chen2013combinatorial} have become an important problem in the online learning. In this paper, we focus on the setting that for a given set of arms $\Acal$ with size $n$ (e.g. $n$ products to be recommended), at every time step $t$, the online learner selects $k$ arms from $\Acal$, and offers the selected set $s$ to the customer. The customer rewards each arm $a_i \in s$ with a set-dependent $X_{i, s}$, and the online learner observes the rewards of each arm. The goal of the online learner is to minimize the regret of not selecting $s^*$ which maximizes the expected reward.

It is observed that a human's preference is typically constructed only when offered a set of alternatives, and the preference can be inconsistent across different sets \citep{macdonald2009preference}. For example, for 3 items $A, B, C$ offered in sets of two, a person can prefer $A$ over $B$, $B$ over $C$ and $C$ over $A$. The loops and reverses in preference motivate us to study the combinatorial bandits setting where the reward distribution of each arm is set-dependent, and crucially, without a total order (\Cref{def:total_order}) in $\Acal$.

\subsection{An old Algorithm, a weak assumption, and a key observation for regret analysis}

Upper Confidence Bound (UCB) algorithm is the standard off-the-shelf choice for many bandit problems. Even in the presence of set-dependent reward, one can nevertheless ignore the set $s$ and maintain UCB estimations for the arms in $\Acal$. In each time step, set $s$ is constructed with the $k$ arms with the highest UCB. The UCB of an arm $a_i\in\Acal$ is defined in the usual way as $\textit{UCB}_i(t) = C_i(t)/N_i(t)+\sqrt{\frac{\alpha\log T}{N_i(t)}}$, where $C_i(t)$ is the cumulative reward of arm $a_i$, $N_i(t)$ is the number of times that arm $a_i$ is in the selected set $s$ up to time $t$, and $\alpha$ is a constant. This is the algorithm we study in this paper (\Cref{alg:UCB}).

Empirically, even in the setting where the reward distribution is set-dependent, people still use the aforementioned UCB algorithm \citep[e.g. the closely-related Sparring algorithm][]{ailon2014reducing}, however, only as a heuristic with little theoretical understanding. \textbf{Existing analysis of UCB does not provide a regret bound in this setting}, as there is no fixed expected reward associated with the arms. In particular, in general, it is impossible to prove any regret bound better than $O(n^k)$ without any additional assumption - since without any additional assumption, the feedback for one set does not give any indication about any other sets.

In this paper, we propose a new assumption for the reward model which we call \textbf{\textit{weak optimal set consistency}} (\Cref{assumption:weak}), under which the UCB algorithm provably achieves small regret. \Cref{assumption:weak} assumes that given the optimal set $s^*$, for any sub-optimal set $s$ and any arm $a$ that is common in $s$ and $s^*$, the reward expectation of $a$ is higher in $s$ than in $s^*$ (since other arms in $s$ are "less competitive"). As the assumption does not constrain the relationship between any two sub-optimal sets, it does not assume any total order for $\Acal$. Examples (see \Cref{example:assumption} and \Cref{subsec:illustrative}) are constructed to show \Cref{assumption:weak} can capture a wide range of set-dependent reward distribution with no total order. Moreover, many previously studied reward models (Multinomial Logit, Random Utility Model, etc.) are special cases of \Cref{assumption:weak} (see discussion in \Cref{subsec:model_and_assumption}).

To build intuition for how the UCB algorithm works under \Cref{assumption:weak}, we present an illustrative experiment. The imaginary environment considers offering suggestions from 6 candidates (shown in \Cref{fig:illustrative_experiment}) to customers looking for cameras, where 3 of them need to be offered each time. The reward model is set such that there is no total order (see \Cref{subsec:illustrative}). The UCB algorithm converges to the optimal set $\cbr{\texttt{Nikon, Canon, Sony}}$, and \Cref{fig:illustrative_experiment} shows the process.

\begin{figure}[ht]
  \begin{minipage}[c]{0.49\textwidth}
    \includegraphics[width=\textwidth]{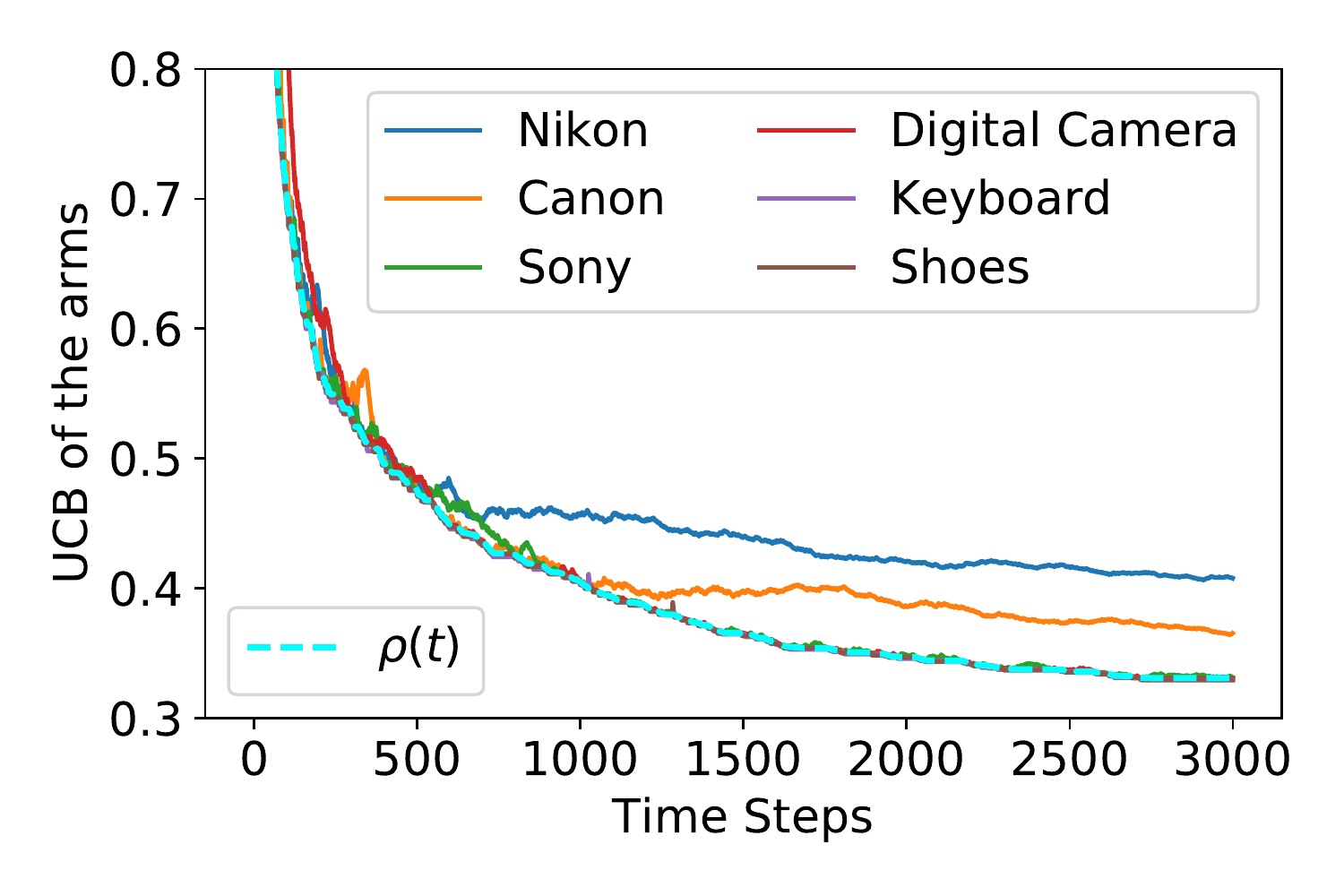}
  \end{minipage}\hfill
  \begin{minipage}[c]{0.5\textwidth}
    \caption{Evolution of UCB in the environment defined in \Cref{subsec:illustrative}. \textbf{Observation:} the UCB of all arms decrease together initially, and the arms in the optimal set separate out later. $\rho(t)$ (defined later) precisely captures this dynamics. Note this happens without the arms having total order or fixed set-independent reward expectation. In fact, the \texttt{Digital Camera} has the highest reward expectation in most sets, while the optimal set is $\cbr{\texttt{Nikon, Canon, Sony}}$.
    } \label{fig:illustrative_experiment}
  \end{minipage}
\end{figure}

Here we formalize the observation in \Cref{fig:illustrative_experiment}. We introduce $\rho(t)$, which helps to characterize the dynamics of UCB.  Let $s(t)$ be the set of arms selected by the aforementioned UCB algorithm at time $t$, $\rho'(t) = \min_{a_i \in s(t)}\textit{UCB}_i(t)$, and $\rho(t) = \min_{\tau \le t}\rho'(\tau)$. By definition, $\rho(t)$ is monotonically non-increasing, and $\textit{UCB}_{i}(t) \ge \rho'(t) \ge \rho(t),~\forall a_{i} \in s(t)$, (i.e. $\rho(t)$ is a lower bound for the UCB of the arms in $s(t)$). The following lemma shows that, for the arms not in $s(t)$, $\rho(t)$ is always an upper bound, and soon a tight estimate of all their UCB (proof in \Cref{sec:algorithm_and_regret}).

\begin{lemma}[Dynamics of UCB]\label[lemma]{ob:dynamics_UCB} $\rho(t) \ge \textit{UCB}_{i}(t) \ge \rho(t)\rbr{1 - \frac{1}{N_{i}(T)}},~\forall a_{i} \notin s(t)$.
\end{lemma}

Notice that, under \Cref{assumption:weak}, once $\rho(t) \leq P(a_i|s^*)$ for some $a_i \in s^*$, all subsequent $s(t)$ will always contain $a_i$, due to the fact that with high probability, $\textit{UCB}_i(t) \ge P(a_i|s^*)$ for all $t\in [T]$ (see \Cref{sec:algorithm_and_regret}). This matches the observation in \Cref{fig:illustrative_experiment}. Further, we can upper bound the time it takes for $\rho(t)$ to be smaller than $P(a_i|s^*)$, which can be converted into a finite time regret bound. We want to emphasize that all the analysis is done without requiring the arms to have set-independent reward expectation (or any notion of intrinsic value), which is drastically different from the standard UCB analysis.

As a summary, our \textbf{main contributions} are:

\begin{itemize}[leftmargin=*, itemsep=0.5ex]
    \item We formalize the combinatorial bandits problem with \textit{weak optimal set consistency} assumption (\Cref{assumption:weak}) which does not require a total order for arms. The new assumption covers many commonly adopted reward models (e.g. Multinomial Logit, and Random Utility Model, etc).
    \item We present a novel analysis of the UCB algorithm (\Cref{alg:UCB}) when the arms do not have set-independent expected reward (or any notion of intrinsic value). Specifically, we prove \Cref{alg:UCB} has a gap-dependent $O(n k^{2} \log T/\epsilon)$ regret upper bound (\Cref{thm:instance_dependent_upper_bound}), as well as a gap-independent $O(k^{2}\sqrt{n T \log T})$ regret upper bound (\Cref{thm:instance_independent_upper_bound}). Here $n$ is the total number of arms, $k$ is the size of selected set $s$, $T$ is the time horizon and $\epsilon$ is the minimum gap between the optimal and sub-optimal set.
    \item Under \Cref{assumption:weak}, we prove a regret lower bound $\Omega(n\log T/k\epsilon)$ when only one of the arms in the selected set has non-zero reward; and a lower bound $\Omega(n \log T / \epsilon)$ when multiple arms in the selected set can have non-zero reward (\Cref{thm:regret_lower_bound}). It demonstrates the optimality of \Cref{alg:UCB} for any constant set size $k$.
\end{itemize}

\section{Motivation and Related Work}\label{sec:related_work}

\begin{table}
\centering
\resizebox{\textwidth}{!}{%
\begin{tabular}{|c|c|c|c|c|}
\toprule
Algorithm & Regret & Fixed $k$ & Set-Dep. Reward & No Total Order  \\ 
\hline
CUCB \citep{chen2013combinatorial} & $ O\rbr{\frac{k^{2}n \log T}{\epsilon}}$ & \cmark & \xmark & \xmark  \\ 
\hline
CombUCB1 \citep{kveton2015tight} & $O\rbr{\frac{k n \log T}{\epsilon}}$ & \cmark & \xmark & \xmark \\ 
\hline
ESCB \citep{combes2015combinatorial} & $O\rbr{\frac{\sqrt{k}n\log T}{\epsilon}}$ & \cmark & \xmark & \xmark \\
\hline
MNL-TS \citep{agrawal2017thompson} & $O\rbr{\sqrt{NT}\log TK}$ & \cmark  & \cmark~(MNL) & \xmark \\
\hline
Explor.-Exploit. \citep{agrawal2019mnl} &  $O\rbr{\frac{k n \log T}{\epsilon}}$ &  \cmark & \cmark~(MNL) & \xmark \\
\hline
MaxMin-UCB \citep{saha2019combinatorial} & $O\rbr{\frac{n\log T}{\epsilon}}$ & \xmark & \cmark~(MNL) & \xmark  \\
\hline
Rec-MaxMin-UCB \citep{saha2019combinatorial} & $O\rbr{\frac{n \log T}{k \epsilon}}$ & \cmark & \cmark~(MNL) & \xmark \\
\hline
Choice Bandits \citep{agarwal2020choice}  & $O \rbr{\frac{n^{2}\log n}{\epsilon^{2}} + \frac{n \log T}{\epsilon^{2}}}$ & \xmark & \cmark & \cmark \\
\hline
\Cref{alg:UCB} (Ours) & $O \rbr{\frac{k^{2}n\log T}{\epsilon}}$ & \cmark & \cmark~  & \cmark \\
\bottomrule
\end{tabular}}\label{tab:related_work}
\caption{Regret upper bounds and settings for stochastic combinatorial bandits. The check marks in "Fixed $k$" mean the algorithms do not need to change the size of $s$ in different time $t$, while cross marks mean they need to change $k$ to achieve small regret. The check marks in "Set-Dep. Reward" mean the reward distribution of arms depends on the set they reside in, while cross marks mean the reward of the arms are generated independent of the set. The cross marks in "No Total Order" mean assuming
individual arms to have intrinsic value, and a total order among the arms, while check marks mean the algorithm does not require such assumption.}
\end{table}

\paragraph{Set-dependent Reward without Arms' Total Order.} 

The inconsistency of human preference \citep{macdonald2009preference} motivates us to study the combinatorial bandit where the reward distribution of each arm depends on the set it resides in, without a total order among the arms. Correspondingly, we propose the \textit{weak optimal set consistency} reward model (\Cref{assumption:weak}), which covers various reward models adopted by many combinatorial bandits work.

The simplest reward model assumes the reward of each arm is generated independent of the selected set (see \Cref{subsec:strict_ordering}) and has been studied in \citep{chen2013combinatorial,kveton2015tight,combes2015combinatorial}. Other work adopt more complicated models to capture the set-dependent reward distribution. However, many of them, on the contrary of \Cref{assumption:weak}, assume a total order among the arms. For example, the Multinomial Logit Model (MNL) assumes a deterministic utility associated with each arm, which induces a total order \citep{abeliuk2016assortment, agrawal2019mnl, saha2019combinatorial, flores2019assortment}. \citet{desir2015capacity,
blanchet2016markov} approximate the user's choice as a random walk on a Markov chain. \citet{berbeglia2016discrete} shows that the discrete choice model and the Markov chain model can be viewed as instances of a "random utility model" (RUM), which also assumes a total order of all the arms. We will show in \Cref{subsec:strict_ordering} that MNL and RUM are both special cases of \Cref{assumption:weak}.

For related work that does not assume total order, \citet{yue2011linear} study linear bandits and assumed a submodular value function which is known to the algorithm. The Choice Bandits \citep{agarwal2020choice} assumes there exists a single best arm that has the largest expected reward in any set, which comes from a different perspective compared with our work.

\paragraph{Fixed Set Size $k$.} 

Our setting requires the size of the selected set $s$ to be exactly $k$. In practice, $k$ represents the available "displaying slots", which should be fully utilized. One common alternative is to require the size of $s$ less than or equal to $k$. However, that alternative usually leads to algorithms that yield set with size strictly less than $k$ most of the time \citep{saha2019combinatorial}. Other related settings \citep{chen2013combinatorial,kveton2015tight,combes2015combinatorial,agrawal2019mnl,agrawal2017thompson} do not allow the algorithm to freely change the size of $s$.

\paragraph{Feedback Model.} There are two commonly studied feedback models. One assumes the online learner only observes the (stochastically) best arm within the set and its reward; the other one assumes each arm generates reward independently, conditioned on the set, and the online learner observes the reward of all arms in the set.

The first feedback model reflects the relative goodness of one arm when comparing with the rest of arms in the set. Such relative feedback has been studied in the dueling bandit problem \citep{yue2012k}, with the focus on relative feedback of 2 arms. Several algorithms have been proposed for the dueling bandits \citep{yue2012k, zoghi2013relative}, while others reduce the dueling bandits to standard multi-arm bandits \citep{ailon2014reducing}. Going beyond 2 arms, the multi-dueling bandits
problem \citep{brost2016multi, sui2017multi} focuses on the pairwise relative feedback which has strictly more information than the single best arm feedback. \citet{saha2018battle, saha2019combinatorial} consider the case where only the best arm in the set is revealed, but focus on recovering the single best arm, instead of the best set.

The second feedback model reveals absolute goodness of the arms within the set, which is more commonly adopted in the stochastic combinatorial bandit problem with semi-bandit feedback \citep{chen2013combinatorial,kveton2015tight,combes2015combinatorial}. Our assumption, algorithm and analysis cover both of the feedback models.

\section{Problem Setup and the Weak Optimal Set Consistency Assumption}\label{sec:setup}

In this section, we first present the combinatorial bandit problem setup and introduce the \textit{weak optimal set consistency} assumption (\Cref{assumption:weak}). We then formally define the "total order" for the arms, and show that many widely studied models (MNL, RUM, etc.) assume such total order and are covered by \Cref{assumption:weak}. We conclude the section with an illustrative example, showing \Cref{assumption:weak} covers non-trivial cases, where there is no total order for $\Acal$.

\subsection{Notations and Definitions}

We consider the stochastic combinatorial multi-armed bandits problem. Given a fixed set of arms $\Acal = \cbr{a_{1}, a_{2}, \cdots, a_{n}}$, let $\Scal$ denote the all $n$-choose-$k$ subsets of $\Acal$. At each time step $t$, the online learner selects a $s(t) \in \Scal$ ($|s(t)| = k$ by definition). The online player then observes the stochastic reward $X_{i, s(t)}$ of all the arms in $s(t)$. To remove ambiguity, we always refer the $a \in \Acal$ as \textbf{\textit{arm}}, and the $s\in \Scal$ as \textbf{\textit{set}}.

The total reward of set $s(t)$ is defined as $Q(s(t)) = \sum_{a_{i} \in s(t)} X_{i, s(t)}$. Let $s^{*}$ be the optimal set, which maximizes the expected reward $\argmax_{s \in \Scal} \EE\sbr{Q(s)}$. The regret is then defined to be
\begin{align*}
  \textit{reg}(t) = \EE\sbr{Q(s^{*}) - Q(s(t))},\quad \text{and} \quad R(T) = \sum_{t=1}^{T}\textit{reg}(t)
,\end{align*}
where the $\textit{reg}(t)$ is the regret at step $t$, and $R(T)$ is the total regret up to $T$. Our goal is to design algorithm for the online player to minimize $R(T)$.

\subsection{Weak Optimal Set Consistency Assumption}\label{subsec:model_and_assumption}

One important feature that distinguishes our setting with standard stochastic combinatorial bandits is the {\textit{set-dependent reward distribution}} and not assuming a total order for the arms.

Here we focus on the binary reward with $X_{i, s} \in \cbr{0, 1}$ and let $P(a_i | s) = \EE\sbr{X_{i, s}}$, with extensions to any bounded reward distribution discussed in \Cref{sec:extension}. Formally, we have the following assumption about $P(a_i | s)$:

\begin{assumption}[\textbf{Weak Optimal Set Consistency}]\label{assumption:weak}
  For any sub-optimal set $s$ and any $a$ that is common in $s, s^*$, we assume $P(a|s) \ge P(a|s^*)$.
\end{assumption}

One salient feature of \Cref{assumption:weak} is \textbf{\textit{not}} assuming the arms $a\in\Acal$ to have total order at any time $t$. We first present several examples that are allowed by our assumption but not other reward models, and formally discuss the "total order" in next subsection.

\begin{example}\label{example:assumption}
  For any $k > 2$, with out loss of generality, we take $a_{1} \in s^{*}, a_{2}\in s^{*}$ with $P(a_{1}|s^{*}) \ge P(a_{2}|s^{*})$, and take $a_{3}, a_{4} \notin s^{*}$. For some sub-optimal set $s_{i}$, \Cref{assumption:weak} allows for:
  \begin{enumerate}
    \item Reversed relative reward expectation:
      \begin{align*}
        &P(a_{1}|s^{*}) \ge P(a_{2}|s^{*}),\quad P(a_{2} | s_{1}) > P(a_{1} | s_{1}),\quad\text{for some $s_{1} \supset \cbr{a_{1}, a_{2}}$} \\
        &P(a_{3} | s_{2}) > P(a_{4} | s_{2}),\quad P(a_{4}| s_{3}) > P(a_{3} | s_{3}), \quad\text{for some $s_{2}, s_{3}$ both containing $a_{3}, a_{4}$}
      .\end{align*}
    \item Non-transitive relative reward expectation: for some $s_4 \supset \cbr{a_2, a_3}, s_5\supset \cbr{a_1, a_3}$,
      \begin{align*}
          P(a_1|s^*) > P(a_2 | s^*),\quad P(a_2 | s_4) > P(a_3 | s_4),\quad P(a_3 |s_5) > P(a_1 | s_5)
      .\end{align*}
  \end{enumerate}
\end{example}
Note that the $s_5$ in the "non-transitive" part of \Cref{example:assumption} also shows that \Cref{assumption:weak} allows the arms not in $s^{*}$ to be better than the arms belonging to $s^{*}$ in some sub-optimal set.

\subsection{{Total Order for Arms and More Restrictive Existing Models}}\label{subsec:strict_ordering}

We start by formally defining the "total order" for arms.

\begin{definition}[Total order for the arms]\label[definition]{def:total_order}
  Given a reward model $P(a|s)$ and any two arms $a_1, a_2 \in \Acal$, we say $a_1 \le a_2$, if $P(a_1 | s) \le P(a_2|s)$ for every $s$ containing $a_1, a_2$. 
  
  Further, a reward model $P(a|s)$ assumes total order for $\Acal$ if: (1) \textit{comparability}, for all $a_1, a_2 \in \Acal$, either $a_1 \le a_2$ or $a_2 \le a_1$; and (2) \textit{transitivity}, $a_1 \le a_2$, $a_2 \le a_3$ implies $a_1 \le a_3$.
\end{definition}

From \Cref{example:assumption}, we see that \Cref{assumption:weak} needs not satisfy either \textit{comparability} or \textit{transitivity} and thus does not assume a total order for $\Acal$. Further, we show that many existing models assume total order for $\Acal$ according to \Cref{def:total_order}, and are special cases of \Cref{assumption:weak}.

\noindent\textbf{Multinomial Logit (MNL):} MNL assumes a deterministic utility $v_{i}$ associated with each $a_i$ and the probability of $a_{i}$ receiving non-zero reward in $s$ is $P(a_{i}| s) = \frac{e^{v_{i}}}{e^{v_{0}} + \sum_{a_{j} \in s}e^{v_{j}}}$, where $v_0$ is a constant. One can verify that the $v_i$s of MNL induce a total order for $\Acal$, and the optimal set $s^{*}$ is composed by arms with highest $v_{i}$. \Cref{assumption:weak} covers MNL as $\sum_{a_{j} \in s}e^{v_{j}} \le \sum_{a_{j} \in s^{*}} e^{v_{j}}$ for any $s \neq s^{*}$.

\noindent\textbf{Random utility model (RUM):} RUM assumes a (random) utility associated for all $a_{i} \in\Acal$, with $U_{i} = v_{i} + \epsilon_{i}$, where $v_{i}$ is a deterministic utility and $\epsilon_{i}$s are i.i.d. random variables drawn from distribution $\Dcal$ at every time step $t$. The probability of $a_{i}$ in $s$ receiving non-zero reward is given by $\PP(a_{i} | s) = \PP\rbr{U_{i} > U_{j}, \forall a_{j}\in s \text{ and } i\neq  j }$. To model the event of no arm $a\in s$ receives non-zero reward, $s$ can be augmented to $s \cup \cbr{a_{0}}$, with random utility $U_{0}$ of $a_0$ defined similarly. When $U_{0}$ is the largest, no arm $a \in s$ receives non-zero reward. It can be verified that $v_i$s in RUM induce a total order for $\Acal$, and the optimal set $s^{*}$ is composed by arms with highest $v_{i}$. For any arm $a\in s^{*}$, putting it to sub-optimal set $s$ leads to $a$ having a larger chance of receiving non-zero reward, as other arms have smaller $v_{i}$, thus satisfies \Cref{assumption:weak}.

\noindent\textbf{Independent reward:} Independent reward model assumes a deterministic reward expectation $v_{i}$ associated with arm $a_{i}$. For the arm $a_{i}$ in any set $s$, it assumes $P(a_{i}|s) = v_{i}$. The $v_i$s immediately induce a total order for $\Acal$. The independent reward model is also covered by \Cref{assumption:weak}, as $P(a_{i}|s)$ does not change in different $s$.

\subsection{An Illustrative Example}\label{subsec:illustrative}

To further build intuition on \Cref{assumption:weak}, we present a synthetic example of providing suggestions to customers looking for cameras. There are 6 candidates \{\texttt{\textcolor{blue}{Nikon, Sony, Canon,}} \texttt{\textcolor{brown}{Digital Camera,}}\texttt{\textcolor{red}{Keyboard, Shoes}}\}. Every time we need to offer 3 suggestions and the customer picks at most one of them. 

\begin{figure}[ht]
    \includegraphics[clip, trim={80 300 150 60}, width=\textwidth]{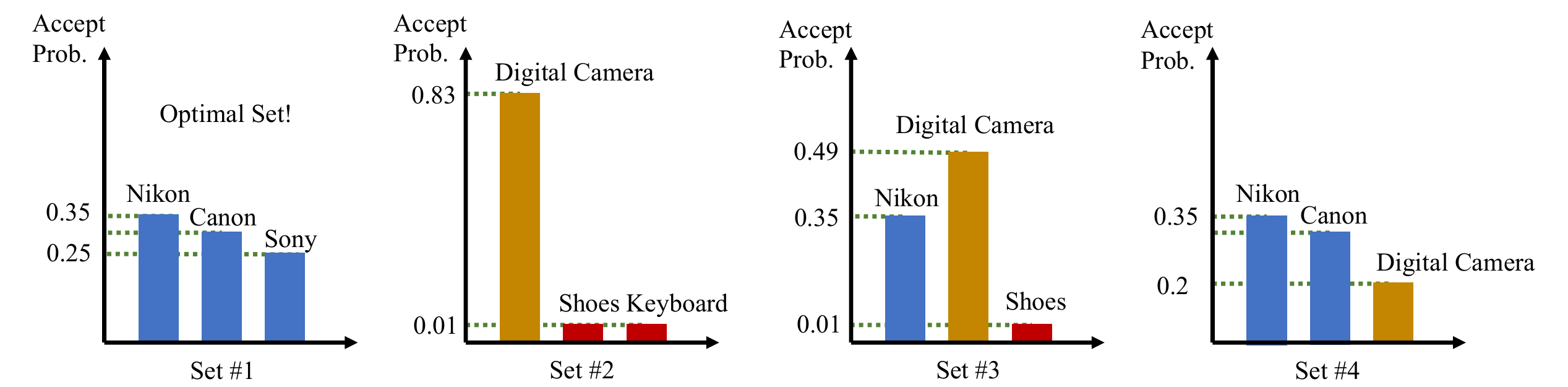}
    \caption{Four representative sets of the example in \Cref{subsec:illustrative}. The set \#1 is optimal, as it maximizes the sum of accepting probability of the suggestions. The \texttt{\textcolor{brown}{Digital Camera}} has highest accepting in many sub-optimal sets (even when paired with the suggestions belonging to the optimal set. see set \#3). Such instances break the total order, but are covered by \Cref{assumption:weak}.}
    \label{fig:illustrative_example}
\end{figure}

For the accepting probability, we set $P(\texttt{\textcolor{blue}{Nikon}}|\cdot) = 0.35, P(\texttt{\textcolor{blue}{Canon}}|\cdot) = 0.3, P(\texttt{\textcolor{blue}{Sony}}|\cdot) = 0.25$, $P(\texttt{\textcolor{brown}{Digital Camera}}|s) = 0.85 - \sum_{a\in s, a\neq{\texttt{\textcolor{brown}{Digital Camera}}}}P(a|s)$ and $P(\texttt{\textcolor{red}{Shoes}}|\cdot) = 0.01$, $P(\texttt{\textcolor{red}{Keyboard}}|\cdot) = 0.01$. We show 4 representative sets in \Cref{fig:illustrative_example}. It can be verified that the optimal set is \{\texttt{\textcolor{blue}{Nikon, Sony, Canon}}\} and this example satisfies \Cref{assumption:weak}, but cannot be covered by any model that assumes a total order (it violates both \textit{comparability} and \textit{transitivity}).

Notice that the existence of the \texttt{\textcolor{brown}{Digital Camera}} suggestion makes the problem harder. We observe that the \texttt{\textcolor{brown}{Digital Camera}} has the highest accepting probability in many sets. This makes \texttt{\textcolor{brown}{Digital Camera}} seemingly the best single suggestion, but it is not part of the optimal set. 
\section{Algorithm and Regret Analysis}\label{sec:algorithm_and_regret}

In this section, we formally describe the algorithm and present its regret bound (both gap-dependent and gap-independent). We also show the sketch of regret analysis, which presents a novel way of proof, without the arms having fixed reward expectation. The analysis follows by characterizing the dynamics of UCB, for whom the intuition has been discussed in \Cref{sec:introduction}. 

\subsection{Algorithm}

Denote $N_i(t)$ to be the number of times that $a_i$ is included in the selected set $s$ up to time $t$, $C_i(t)$ to be the cumulative reward of arm $a_i$ at time $t$. We have the algorithm shown in \Cref{alg:UCB}.

\Cref{alg:UCB} extends the standard $\alpha$-UCB algorithm. It selects a set of arms with top-$k$ UCB in each steps. It is worth noting that \Cref{alg:UCB} only keeps cumulative reward of the arms in $\Acal$, without accounting for any set-dependent information. Though it may seem contradictory to the set-dependent reward distribution, we will show that \Cref{alg:UCB} leads to small regret.

\begin{algorithm}[h]
\centering 
\begin{algorithmic}[1]
\STATE \textbf{Task:} Given $\Acal$, minimize the regret of not selecting the best $n$-choose-$k$ subset of $\Acal$
\STATE \textbf{Input:} arm set $\Acal$, set size $k$,time horizon $T$
\STATE \textbf{Parameter:} A problem independent constant $\alpha$. Normally set to $2$
\STATE \textbf{Initialize:} $\textit{UCB}_i(1) = \textit{INF}$, $N_i(1) = 0$, $C_i(1)= 0$ for all arm $a_i\in\Acal$
\FOR{$t=1$ \textbf{to} $T$}
\STATE Construct set $s(t)$ with arms that have top-$k$ $\textit{UCB}_i(t)$, ties break randomly. For all $a_i\in s(t)$, Set $N_i(t+1) =N_i(t) + 1$
\STATE Observe feedback. Set $C_i(t+1) = C_i(t)+X_{i, s(t)}$ 
\STATE $\textit{UCB}_i(t+1) = C_i(t+1)/N_i(t+1)+\sqrt{\alpha\log T/N_i(t+1)}$, for all arm $a_i \in s(t)$, and
$\textit{UCB}_{i}(t+1)$ = $\textit{UCB}_{i}(t)$, for all other arms.
\ENDFOR 
\end{algorithmic}
\caption{\textsc{UCB for Combinatorial Bandits without Total Order for Arms} }
\label{alg:UCB}
\end{algorithm}

\subsection{Regret Bound}
We first present the gap-dependent regret bound. Let $\epsilon = \min_{s \neq s^{*}} \EE\sbr{Q(s^{*}) - Q(s)}$ denote the minimum gap in expected reward between the optimal set $s^{*}$ and any sub-optimal set $s$. Recall that $k$ is the size of the selected set $s$, and $n$ is the size of $\Acal$.

\begin{theorem}[Gap-dependent regret upper bound]\label{thm:instance_dependent_upper_bound} For combinatorial bandits problem under \Cref{assumption:weak}, run \Cref{alg:UCB} with parameter $\alpha \ge 2$, we have
    \begin{align*}
        R(T) \le \frac{8 \alpha k^{\frac{3}{2}} n \log T}{\epsilon} + \frac{30 \alpha k^{2} n \log T}{\epsilon} + n = O\rbr{\frac{\alpha k^{2} n \log T}{\epsilon}}.
    \end{align*}
\end{theorem}

Due to the combinatorial nature of $\Scal$, we might see extremely small $\epsilon$. As complementary to \Cref{thm:instance_dependent_upper_bound}, we present the following gap-independent regret bound which holds for any $\epsilon$.

\begin{theorem}[Gap-independent regret upper bound]\label{thm:instance_independent_upper_bound}
  For combinatorial bandits problem under \Cref{assumption:weak}, run \Cref{alg:UCB} with parameter $\alpha \ge 2$, we have
   \begin{align*}
     R(T) \le 2\sqrt{\alpha kn T \log T} + 15k^2\sqrt{\alpha n T \log T}= O\rbr{k^{2}\sqrt{\alpha n T \log T}}.
   \end{align*}
\end{theorem}

\subsection{Proof Sketch}

We first prove \Cref{ob:dynamics_UCB}. For any time step $t$, Recall $\rho'(t) = \min_{a_{i} \in s(t)}\textit{UCB}_{i}(t)$, and $\rho(t) = \min_{s \le t}\rho'(s)$. \Cref{ob:dynamics_UCB} claims that
  \begin{align*}
    \rho(t) \ge \textit{UCB}_{i}(t) \ge \rho(t)\rbr{1 - \frac{1}{N_{i}(T)}},\quad \forall a_{i} \notin s(t)
  .\end{align*}
\begin{proof}
  For any arm $a_{i} \notin s(t)$, let $t' \le t$ to be the last time step that $a_{i} \in s(t')$. We then have
  \begin{align*}
    C_{i}(t') + \sqrt{\alpha N_i(t') \log T } \ge \rho'(t')N_{i}(t') \ge \rho(t')N_{i}(t') \ge \rho(t)N_{i}(t')
  .\end{align*}
  The last step holds as $\rho(t)$ is non-increasing. With $C_{i}(t) \ge C_{i}(t')$ and $N_{i}(t) = N_{i}(t') + 1$, we have
  \begin{align*}
    C_{i}(t) + \sqrt{\alpha N_{i}(t) \log T } \ge \rho(t)\rbr{N_{i}(t) - 1}
  .\end{align*}
  Dividing both side by $N_{i}(t)$ gives the second inequality. It left to show $\rho(t) \ge \textit{UCB}_{i}(t),~\forall a_{i}\notin s(t)$. Let $t'' \le t$ be the last time step $\rho(t'') = \rho(t)$. It implies 
  \begin{align*}
    \rho'(\tau) > \rho(t'') = \rho(t) \ge \textit{UCB}_{i}(t''), \quad\forall \tau\in(t'', t], a_{i}\notin s(t'')
  .\end{align*} Notice that $\textit{UCB}_{i}(\tau + 1) = \textit{UCB}(\tau)$ if $a_{i} \notin s(\tau)$.
  Therefore for any $a_{i} \notin s(t'')$, it implies $a_{i}\notin s(\tau), \forall \tau\in[t'', t]$. Since there are $n-k$ arms not in $s(t)$ and same number of arms not in $s(t'')$, we have $a_{i} \notin s(t'') \iff a_{i} \notin s(t)$. Thus
    \begin{align*}
      \textit{UCB}_{i}(t) = \textit{UCB}_{i}(t'') \le \rho'(t'') = \rho(t),~\forall a_{i} \notin s(t).
    \end{align*}
  This completes the proof.
\end{proof}

With loss of generality, we assume $s^{*} = \cbr{a_{1}, a_{2},\cdots,a_{k}}$ with $P(a_{1}|s^{*}) \ge P(a_{2}|s^{*})\ge \cdots \ge P(a_{k}|s^{*})$. Let time $t_{l}$ be the last time we have $\rho(t_{l}) \ge \PP(a_{l}|s^{*})$ for $l \le k$, we have the following corollary of \Cref{ob:dynamics_UCB}.
\begin{corollary}\label[corollary]{coro:arm-selection}
  For all time steps $t$ after $t_{l}$, we have $\cbr{a_{1}, a_{2},\cdots, a_{l}} \subset s(t)$.
\end{corollary}

\Cref{coro:arm-selection} shows that after the time step $t_l$, at which $\rho(t)$ falls below $P(a_l | s^*)$, then all subsequent $s(t)$ will always include $\cbr{a_1, \cdots, a_l}$. The next lemma shows the key to bound $t_l$.

\begin{lemma}\label[lemma]{coro:sum}
    For the time step $t_l$, we have
  \begin{align}\label{eq:coro-sum}
    2\sqrt{\alpha k n t_{l} \log T} \ge kt_{l} P(a_{l} | s^{*}) - \sum_{t=1}^{t_{l}}\sum_{i=1}^{n}P(a_{i}|s(t)) - n P(a_{l}|s^{*})
  .\end{align}
\end{lemma}
\begin{proof}
  By \Cref{coro:validity_of_UCB}, we have $2\sqrt{\alpha N_{i}(t_{l})\log T } \ge N_{i}(t_{l}) \textit{UCB}_{i}(t_{l}) - \sum_{t=1}^{t_{l}}P(a_{i}|s(t))$ for all $a_{i}\in\Acal$ with high probability. Combining with \Cref{ob:dynamics_UCB} and summing for all $i \in [n]$ give the desired inequality, with left-hand side follows from $2\sqrt{\alpha k n t_{l} \log T} \ge \sum_{i=1}^{n} 2\sqrt{\alpha N_{i}(t_{l})\log T}$ by Cauchy-Schwarz inequality.
\end{proof}

Intuitively, the left-hand side of \Cref{eq:coro-sum} scales as $\Theta(\sqrt{t_{l}})$ and the right-hand side scales as $\Theta(t_{l})$. Therefore it can be used to upper bound $t_{l}$. However, the second term on the right-hand side of \Cref{eq:coro-sum} has minus sign before it, which requires a more careful analysis.

Based on a stronger version of \Cref{coro:sum} (see \Cref{lemma:strong-sum}), we can bound the number of times that a sub-optimal $s$ is selected before $t_{l}$. Let $t_{l}'$ be the number of times that $s^{*}$ is selected before $t_{l}$. 
\begin{lemma}[Bound the times of selecting sub-optimal set]\label[lemma]{lemma:suboptimal_set_upper_bound}
  We can bound $t_{l} - t_{l}'$ as,
  \begin{align*}
    t_{l} - t_{l}' \le \frac{40\alpha l k n \log T}{\rbr{\Delta_{l} + \epsilon}^{2}}, \text{ if $\Delta_{l} \ge \frac{\epsilon}{10}$;\quad and \quad}t_{l} - t_{l}' \le \frac{40\alpha l k n \log T}{\epsilon^{2}}, \text{ otherwise}
  ,\end{align*}
  where $\Delta_{l}\coloneqq \sum_{i=l}^{k}\sbr{P(a_{l}|s^{*}) - P(a_{i}|s^{*})}$.
\end{lemma}
The next lemma connects regret $R(T)$ to $t_{l} - t'_{l}$ for $l \le k$.
\begin{lemma}[Regret decomposition]\label[lemma]{lemma:regret_decomposition}
  For the regret at time $T$, we have
  \begin{align*}
    R(T) \le 2\sqrt{\alpha k n (t_{k} - t'_{k})\log T} + \sum_{l=1}^{k-1}\delta_{lk}\rbr{t_{l} - t'_{l}} + nP(a_{k}|s^{*})
  .\end{align*}
  where $\delta_{ij} \coloneqq P(a_{i} | s^{*}) - P(a_{j}| s^{*})$.
\end{lemma}
Now we are ready to prove \Cref{thm:instance_dependent_upper_bound}, which gives the gap-dependent regret bound.
\begin{proof}
  Combining \Cref{lemma:suboptimal_set_upper_bound,lemma:regret_decomposition}, we have
  \begin{align*}
    R(T) \le \frac{8 \alpha k^{\frac{3}{2}} n \log T}{\epsilon} + \sum_{l=1}^{k-1}\delta_{lk}\rbr{t_{l} - t'_{l}} + nP(a_{k}|s^{*})
  .\end{align*}
  Directly applying \Cref{lemma:suboptimal_set_upper_bound} to the summation leads to a $O(k^{3})$ term. To obtain the $O(k^{2})$, we can use the \Cref{lemma:delta-times-t}, which gives
  \begin{align*}
    \sum_{i=1}^{k-1}\delta_{ik}\rbr{t_{i} - t'_{i}} \le \frac{30 \alpha k^{2} n \log T}{\epsilon}
  .\end{align*}
  Combining the two inequalities gives the $O\rbr{\frac{\alpha k^{2}n \log T}{\epsilon}}$ regret bound.
\end{proof}
The proof of \Cref{thm:instance_independent_upper_bound} follows by discussing the relationship between $\Delta_{i} + \epsilon$ and $k\sqrt{\frac{\alpha n \log T}{T}}$.
\begin{proof}
  Recall that $\delta_{ik} = P(a_{i}|s^{*}) - P(a_{k}|s^{*})$, and $\Delta_{l} = \sum_{i=l}^{k}\delta_{li}$. Let $m$ denote the largest $i \in [0, k]$ such that $\Delta_{i} + \epsilon \ge 10 k \sqrt{\frac{\alpha n \log T}{T}}$. Further note that a trivial bound for all $t_{l} - t'_{l}$ is $T$. Combining \Cref{lemma:suboptimal_set_upper_bound,lemma:regret_decomposition}, we have
  \begin{align*}
    R(T) \le 2\sqrt{\alpha k n T \log T} + \frac{50 \alpha k^{3} n \log T}{\Delta_{m} + \epsilon} + \sum_{i=m+1}^{k}\delta_{ik}T
  .\end{align*}
  By definition of $\Delta_{m+1}$, we have $\delta_{ik} \le \Delta_{m+1}, \forall i \ge m+1$. Therefore, we have
  \begin{align*}
    R(t) \le 2\sqrt{\alpha kn T \log T} + \frac{50 \alpha k^{3} n \log T}{\Delta_{m} + \epsilon} + (k - m)\Delta_{m + 1} T.
  \end{align*}
  With $\Delta_{m} + \epsilon \ge 10 k \sqrt{\frac{\alpha n \log T}{T}} \ge \Delta_{m + 1} + \epsilon$, we have the desired regret bound.
\end{proof}

\section{Regret Lower Bound}\label{sec:regret_lower_bound}

We present the regret lower bound under \Cref{assumption:weak}. In particular, we distinguish two reward models with 1) $\Mcal1$, that allows at most 1 of the arms in the selected set $s$ to have non-zero reward (this includes the RUM and MNL model); and 2) $\Mcal2$, that allows multiple arms to have non-zero reward (this includes the independently generated reward). Both $\Mcal1$, $\Mcal2$ are covered by \Cref{assumption:weak}, but the lower bounds differ by a factor of $k$.

\begin{theorem}[Regret Lower Bound]\label{thm:regret_lower_bound}
  For any online learning algorithm, there exists an environment instance with reward model $\Mcal1$ and satisfies \Cref{assumption:weak}, such that the algorithm induces a regret of $R(T) = \Omega\rbr{\frac{n \log T}{k\epsilon}}$. There exists another environment instance with reward model $\Mcal2$ and satisfies \Cref{assumption:weak}, such that the algorithm induces a regret of $R(T) = \Omega\rbr{\frac{n \log T}{\epsilon}}$.
\end{theorem}

\begin{proof}
  We defer the detailed proof to \Cref{proof:regret_lower_bound} and highlight the reason for the difference in $k$ here. Intuitively, for two different environments $\Ecal_{1}, \Ecal_{2}$, one need to select the sets that have different reward distribution in $\Ecal_{1}, \Ecal_{2}$ to accumulate enough "information" (KL-Divergence) to distinguish the two environments.

  Now consider two distributions $p, q \in \RR^{k}_{+}$, which are the reward expectations of all arms in set $s$ under environment $\Ecal1$ and $\Ecal2$. Each element of $p, q$ corresponds to one arm in $s$. For simplicity, let $p_{1}$ and $p_{2}$ be the two smallest elements in $p$, and $q$ differs from $p$ as $q_{1} = p_{1} + \epsilon$ and $q_{2} = p_{2} - \epsilon$. One can show that $D_{KL}(p, q) = \frac{\epsilon^{2}}{p_{1}} + \frac{\epsilon^{2}}{p_{2}} + o(\epsilon^{2})$.

  Under feedback model $\Mcal1$, as the rewards are mutually exclusive, we need $\sum_{i=1}^{k}p_{i} \le 1$. It implies that $p_{1}$ and $p_{2}$ are smaller than $\frac{1}{k-1}$. Whereas for feedback model $\Mcal2$, we can set $p_{1} = p_{2} = \frac{1}{2}$. Therefore playing one sub-optimal set in $\Mcal1$ typically brings $k$-times larger "information" than in $\Mcal2$, which means one can distinguish $\Ecal_{1}$ and $\Ecal_{2}$ by selecting ${k}$-times less sets in $\Mcal1$.
  This brings the difference in the regret lower bound.

\end{proof}

The dependency of $n, B, T, \epsilon$ in the lower bound matches the upper bound (\Cref{thm:instance_dependent_upper_bound}). \Cref{alg:UCB} is thus near-optimal for constant set size $k$ for both $\Mcal1$ and $\Mcal2$, under \Cref{assumption:weak}.

There is a gap on $k$ for between \Cref{thm:instance_dependent_upper_bound} and \Cref{thm:regret_lower_bound}. The gap on $k$ also shows up under the stronger MNL assumption \citep{agrawal2019mnl}. There exists several stronger lower bounds in previous work. By allowing the size of set to change (instead of fixing the size to $k$ as ours), the lower bound can be improved to be $k$-independent for $\Mcal1$ \citep{chen2017note}; with a differently defined $\Scal$, a lower bound that linearly scales with $k$ can be obtained for $\Mcal2$ \citep{kveton2015tight}. Those results are not directly comparable with ours for the difference in settings.

We believe our lower bound can potentially be improved, since the arms still have a total order in our environment construction for lower bound analysis, which implies that \Cref{thm:regret_lower_bound} does not fully capture the hardness of our setting (under \Cref{assumption:weak}).

\section{Beyond Binary Reward}\label{sec:extension}

In previous sections, we focus on the setting with $X_{i, s} \in \cbr{0, 1}$. Here we extend the reward distribution to any bounded distribution. With a minor change in \Cref{alg:UCB}, it achieves the same regret bound as in \Cref{thm:instance_dependent_upper_bound,thm:instance_independent_upper_bound}.

\subsection{Extended Problem Setting and Assumption}

We keep all previous settings but the reward distribution the same. For any set $s$, the reward $X_{i, s}$s are now generated from any bounded distribution with $ {X_{i, s}} \in [0, B]$, and the online learner observes all rewards $X_{i, s}$. Correspondingly, we extend the \textit{weak optimal set consistency} assumption.

\begin{assumption}[\textbf{Extended Weak Optimal Set Consistency}]\label{assumption:extended}
  For any sub-optimal set $s$ and any $a$ that is common in $s, s^*$, we assume $\EE\sbr{X_{i, s}} \ge \EE\sbr{X_{i, s^{*}}}$.
\end{assumption}

\subsection{Algorithm and Regret Upper Bound}
For the extended setting, we can simply modify the $\textit{UCB}_{i}$ update of \Cref{alg:UCB} to 
\begin{align*}
  \textit{UCB}_{i}(t + 1) = C_i(t+1)/N_i(t+1)+B \sqrt{\alpha\log T/N_i(t+1)}
.\end{align*}
The new $\textit{UCB}_{i}$ update provides valid upper bound in the extended setting, as the new reward distributions conditioned on the set $s(t)$ are all sub-Gaussian with parameter $B$. As an immediate corollary of \Cref{thm:instance_dependent_upper_bound,thm:instance_independent_upper_bound}, we have
\begin{corollary}\label[corollary]{coro:extension}
  For combinatorial bandits problem with feedback model under \Cref{assumption:extended}, run the modified \Cref{alg:UCB} with parameter $\alpha \ge 2$, we have
   \begin{align*}
     R(T) = O\rbr{\min\rbr{\frac{\alpha k^{2}B^{2} n \log T}{\epsilon}, k^{2}B\sqrt{\alpha n T \log T}}}
   \end{align*}
\end{corollary}

\section{Experiments}\label{sec:experiments}

\begin{figure}[t]
    \includegraphics[width=\textwidth]{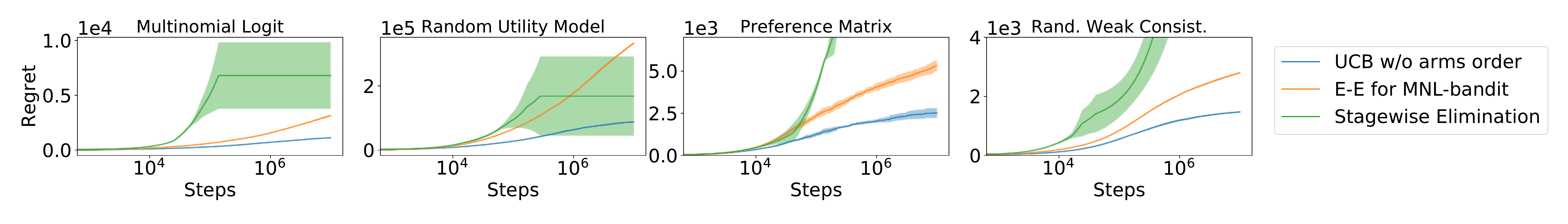}
    \caption{Synthetic experiments with different reward models. The curves are the average of 5 independent runs, with the shaded area representing the standard deviation. The "UCB w/o arms order" corresponds to \Cref{alg:UCB} with $\alpha=2$. "E-E MNL-bandit" refers to the "Exploration-Exploitation algorithm for MNL-Bandit" \citep{agrawal2019mnl}. "Stagewise Elimination" was proposed in \citep{simchowitz2016best}. The parameters are specified as suggested in the original papers.}
    \label{fig:sim}
\end{figure}

We empirically evaluate the performance of \Cref{alg:UCB} on environments with different reward models (see \Cref{fig:sim}), which shows the broad applicability of our proposed algorithm. We summarize the environments below, with details provided in \Cref{APDX:exp_setup}. 

\noindent\textbf{Multinomial Logit}: Each arm $a_i$ has a intrinsic value $v_i$ and the MNL model is used to determine the reward probability. The total number of arms is set to $n=20$ and the set size is set to $k=10$. The number of possible sets is $184756$.
    
\noindent\textbf{Random Utility Model}: Each arm $a_i$ has an intrinsic utility $v_i$. In every step, the random utility $U_i$ of all arms in the set $s$ are independently generated with mean $\mu_i$ and unit variance from Gaussian distribution. The arm with largest random utility $U_i$ receives the reward. The total number of arms is set to $n=20$ and the set size is set to $k=5$. The number of possible sets is $15504$.
    
\noindent\textbf{Preference Matrix}: We set the total number of arms to $n=10$ and the set size to $k=2$, then directly specify a 10-by-10 preference matrix $M$ to determine the probability of an arm receiving reward. In particular, we set the matrix such that there is no total order for the arms.
    
\noindent\textbf{Random Weak Optimal Set Consistency}: We randomly generate the environment that satisfies Assumption \ref{assumption:weak} via rejection sampling. We set the total number of arms to $n=10$ and the set size to $k=5$. Notice that, these randomly generated environments need not to satisfy the assumption of MNL model (or RUM) other than Assumption \ref{assumption:weak}.

Along with \Cref{alg:UCB}, we also take "E-E for MNL-bandit" (Exploration-Exploitation algorithm for MNL, \citep{agrawal2019mnl}) and "Stagewise Elimination" \citep{simchowitz2016best} for comparisons, which are designed for "Multinomial Logit" and "Random Utility Model" environment. The algorithms are tested in the environments listed above, with results shown in \Cref{fig:sim}.

"E-E for MNL-bandit" and "Stagewise Elim" perform relatively good in the environments that they are designed for. Note that in the "Preference Matrix" environment and "Random Weak Optimal Set Consistency" environment, there is no total order among the arms. The "Stagewise Elimination" falsely eliminates an arm that belongs to the optimal set (due to model mis-specification), and therefore suffers from linear regret. \Cref{alg:UCB} performs better in all the testing environments.

\section*{Acknowledgement}
This work is supported in part by NSF grants 1564000 and 1934932.

\bibliography{ref}
\bibliographystyle{abbrvnat}

\newpage
\appendix

\section{Technical Results}
\begin{lemma}[Validity of Upper Confidence Bound]\label[lemma]{lemma:validity_of_UCB}
    Denote $P_i(t) = P(a_i|s(t))$. For the probability measure generated by all sequences of assortments and reward up to time $T$, we have
    \begin{align*}
         P\rbr{\abs{C_i(t)-\sum_{c=1}^tP_i(c)} \ge \sqrt{\alpha N_i(t)\log T} } \le \frac{2}{T^{2\alpha}},\quad \forall t\le T,\forall i\in [n].
    \end{align*}
\end{lemma}
\begin{proof}
	Consider the quantity
	\begin{align*}
		D_i(t) = C_i(t) - \sum_{c=1}^tP_i(c) 
	\end{align*}
	It is not hard to see that $D_i(0)$ to $D_i(T)$ is a martingale. By Azuma's inequality, we have
	\begin{align*}
		P(D_i(t)\ge d) \le \exp(-2d^2/N_i(t)) \quad P(D_i(t)\le -d) \le \exp(-2d^2/N_i(t))
	\end{align*}
	This comes from the fact that at each time step, if $i$ is selected, the corresponding difference is bounded by 1.
	Equivalently, we have
	\begin{align*}
		P(D_i(t)\ge \sqrt{\alpha N_i(t)\log T }) \le \rbr{\frac{1}{T}}^{2\alpha} \quad P(D_i(t)\le -\sqrt{\alpha N_i(t)\log T}) \le \rbr{\frac{1}{T}}^{2\alpha}
	\end{align*}
	Therefore, we conclude that
	\begin{align*}
		\forall t\le T,\forall i\in [n]\quad P\rbr{\abs{C_i(t)-\sum_{c=1}^tP_i(c)} \ge \sqrt{\alpha N_i(t)\log T} } \le \frac{2}{T^{2\alpha}}
	\end{align*}
\end{proof}

\begin{corollary}[Corollary of \Cref{lemma:validity_of_UCB}]\label[corollary]{coro:validity_of_UCB}
  For all time step $t \in [T]$, and all arm $a_{i} \in \Acal$, we have
  \begin{align*}
    2\sqrt{\alpha N_{i}(t)log T } \ge N_{i}(t) \textit{UCB}_{i}(t) - \sum_{c=1}^{t}P(a_{i}|s(c))
  .\end{align*}
\end{corollary}

\begin{lemma}\label[lemma]{lemma:delta-times-t}
  Recall that we assumed $a_1, \cdots, a_k$ all belong to $s^*$, with $P_i^* = P(a_i|s^*)$ for $i \in [k]$, and $P_1^* > P_2^* > \cdots > P_k^*$. Recall $\delta_{ij} = P_i^* - P_j^*$ and $\Delta_l = \sum_{i=l}^k\delta_{li}$. Let $t_l$ be the last time with $\rho(t) \ge P_l^*$, and $t_l'$ be the number of times that the optimal set $s^*$ is played. For any $l \le k$, we have
  \begin{align*}
    \sum_{i=1}^{l-1}\delta_{il}(t_{i} - t_{i}') \le \frac{30\alpha l k n \log T}{\Delta_{l - 1} + \epsilon}
  .\end{align*}
\end{lemma}
\begin{proof}
  Denote $l'$ to be the largest $i$ with $\Delta_{i} \ge \epsilon / 10$. Using Lemma \ref{lemma:suboptimal_set_upper_bound}, for $l \le l'$, we have
  \begin{align*}
      \sum_{i=1}^{l-1}\delta_{il}(t_{i} - t_{i}') \le \frac{10\alpha l k n \log T}{\Delta_{l} + \epsilon}\cdot\sum_{i=1}^{l-1}\frac{4\delta_{il}\rbr{\Delta_l + \epsilon}}{\rbr{\Delta_i + \epsilon}^2} \le \frac{20\alpha l k n \log T}{\Delta_l + \epsilon},
  \end{align*}
  where the last inequality follows from Lemma \ref{lemma:bound_sigma}.
  For $l > l'$, we have
  \begin{align*}
       \sum_{i=1}^{l-1}\delta_{il}(t_{i} - t_{i}') & \le \frac{10\alpha l k n \log T}{\Delta_{l} + \epsilon}\cdot\rbr{\sum_{i=1}^{l'}\frac{4\delta_{il}\rbr{\Delta_l + \epsilon}}{\rbr{\Delta_i + \epsilon}^2}  + \sum_{i = l'+1}^{l-1}\frac{4 \delta_{il}\rbr{\Delta_l + \epsilon}}{\epsilon^2}}\\
       & \le \frac{15\alpha l k n \log T}{\Delta_l + \epsilon}\cdot\rbr{\sum_{i=1}^{l'}\frac{4\delta_{il}\rbr{\Delta_l + \epsilon}}{\rbr{\Delta_i + \epsilon}^2}  + \sum_{i = l'+1}^{l-1}\frac{4 \delta_{il}\rbr{\Delta_l + \epsilon}}{\rbr{\Delta_i + \epsilon}^2}}\\
       & \le \frac{30\alpha l k n \log T}{\Delta_i + \epsilon}.
  \end{align*}
  The second inequality follows from $\Delta_l \le \epsilon / 10$ for $l > l'$, and the last inequality follows from Lemma \ref{lemma:bound_sigma}.
\end{proof}

\section{Proof for \Cref{sec:algorithm_and_regret}}

\subsection{Supporting Lemmas}

\begin{lemma}[Stronger version of \Cref{coro:sum}]\label[lemma]{lemma:strong-sum}
  For simplicity, denote $P_{1}^{*} = P(a_{1}|s^{*}), \cdots P_{k}^{*} = P(a_{k}|s^{*})$, and $P(a_{i}|s(t)) = P_{i}(t)$.
  Let $\delta_{ij} = P(a_i | s^*) - P(a_j|s^*)$. For any $t\in [T]$ and any $l \le k$, recall that $t_{l}$ is the last time step with $\rho(t_{l} \ge P_{l}^{*})$, we have
  \begin{align*}
  \sqrt{4\alpha kn \ln T \rbr{t_{l} - \frac{l}{k}t_{l}'}} \ge \sum_{i=1}^{l}P_{i}^{*}t_{l} + (k-l)P_{l}^{*}t_{l} - \sum_{i=1}^{n}\sum_{c=1}^{t_{l}}P_{i}(c) - \sum_{i=1}^{l-1}\delta_{il}(t_{i} - t_{i}') - nP_{l}^{*}
  .\end{align*}
\end{lemma}
\begin{proof}
  By \Cref{coro:validity_of_UCB} and \Cref{ob:dynamics_UCB}, at time $t_{l}$, we have
  \begin{align*}
    2\sqrt{\alpha N_{i}(t_{l})\log T} \ge N_{i}(t)P_{l}^{*} - \sum_{t=1}^{t_{l}}P(a_{i}|s(t)) - P_{l}^{*}
  \end{align*}
  Summing up for $i \ge l + 1$, we have
  \begin{align*}
    2\sum_{i = l + 1}^{n} \sqrt{\alpha \ln T N_{i}(t_{l})} & \ge \sum_{i = 1}^{l}P_{i}^{*} N_{i}(t_{l}) + \sum_{i > l}^{n}P_{l}^{*}N_{i}(t_{1}) - \sum_{i = 1}^{n}\sum_{c = 1}^{t_{l}}P_{i}(c) - n P_{l}^{*} \\
                                                       & \ge \sum_{i = 1}^{l}P_{i}^{*} t_{l} + (k - l) P_{l}^{*} t_{l} - \sum_{i=1}^{n}\sum_{c=1}^{t_{l}}P_{i}(c) - \sum_{i = 1}^{l-1}\delta_{il}\rbr{t_{i} - t'_{i}} - nP_{l}^{*}
  .\end{align*}
  The first inequality follows from $P_{i}(c) \ge P_{i}^{*}$ for any $c$ and $i \le l$, by \Cref{assumption:weak}. The second inequality follows from $t_{l} - N_{i}(t_{l}) \le t_{i} - t_{i}'$, by \Cref{coro:arm-selection}. The desired inequality follows by Cauchy-Schwart inequality
  \begin{align*}
    \sqrt{4 \alpha k n \log T \rbr{k_{l} - \frac{l}{k}t_{l}'}} \ge 2\sum_{i=l+1}^{n}\sqrt{\alpha \log T N_{i}(t_{l})}
  .\end{align*}
\end{proof}

\begin{lemma}\label[lemma]{lemma:bound_sigma}
  Recall that we assumed $a_1, \cdots, a_k$ all belong to $s^*$, with $P_i^* = P(a_i|s^*)$ for $i \in [k]$, and $P_1^* > P_2^* > \cdots > P_k^*$. Recall $\delta_{ij} = P_i^* - P_j^*$ and $\Delta_l = \sum_{i=l}^k\delta_{li}$. Let $\sigma_{ij} = \frac{4\delta_{ij}(\Delta_j + \epsilon)}{(\Delta_i + \epsilon)^2}$, we have
  \begin{align*}
      \sum_{j=i}^{k} \sigma_{ij} \le 2,~\forall i \le k, \forall \epsilon \ge 0.
  \end{align*}
\end{lemma}
\begin{proof}
  Expanding the summation, we have
  \begin{align*}
      \sum_{j=i}^k \sigma_{ij} & = \sum_{j=i}^k \frac{4\delta_{ij}(\Delta_j + \epsilon)}{\rbr{\Delta_i + \epsilon}^2} = 4\sum_{j=i}^k\frac{\delta_{ij}}{\Delta_i + \epsilon}\rbr{\sum_{m=j}^k \frac{\delta_{mk}}{\Delta_i + \epsilon} + \frac{\epsilon}{\Delta_i + \epsilon}}.
  \end{align*}
  Note that 
  \begin{align*}
      \sum_{m=j}^k \delta_{mk} + \sum_{m=i}^{j}\delta_{im} \le \Delta_i + \epsilon \implies \sum_{m=j}^k \frac{\delta_{mk}}{\Delta_i + \epsilon} \le 1 - \sum_{m=i}^j \frac{\delta_im}{\Delta_i + \epsilon}.
  \end{align*}
  For brevity, let $x_m = \frac{\delta_{im}}{\Delta_i + \epsilon}$, we have $\sum_{m=i}^k x_m \le 1$ and 
  \begin{align*}
      \sum_{j=i}^k\sigma_{ij} \le 4\sum_{j=i}^k x_j (1 - \sum_{m=i}^jx_m) \le 2.
  \end{align*}
  The last inequality holds for any $\sum_{m=i}^k x_m \le 1$.
\end{proof}
\begin{lemma}\label[lemma]{lemma:f}
    For any $1 \le i < j \le k$, define funciton $f(i, j) = 0.4\sigma_{ij} + \sum_{m = i+1}^{j-1}0.4\sigma_{im}f(m, j)$. We have
    \begin{enumerate}
        \item $f(i, j) = 0.4\sigma_{ij} + \sum_{m=i+1}^{j - 1}0.4 f(i, m)\sigma_{mj}$
        \item $f(i, j) \le 1$
    \end{enumerate}
\end{lemma}
\begin{proof}
  We first prove the first part. Let $\Pi(i, j)$ be the power set of $\cbr{i, i+1, \cdots, j-1, j}$. Let $\Gamma(i, j) = \cbr{x| x\in \Pi(i,j), i \in x, j \in x}$. Further, for $x \in \Gamma(i, j)$defining
  \begin{align*}
      g(x) = \sigma_{x_1, x_2} \cdot \sigma_{x_2, x_3} \cdots \sigma_{x_{\abs{x} - 1}, x_{\abs{x}}}.
  \end{align*}
  For example, for $x = \cbr{2,3, 5, 7}$, we have $g(x) = \sigma_{23}\cdot\sigma_{35}\cdot\sigma_{57}$. By definition, it can be shown via induction that
  \begin{align*}
      f(i, j) = \sum_{x \in \Gamma(i, j)}0.4^{\abs{x}} g(x),
  \end{align*}
  which is equivalent to the first equation in Lemma \ref{lemma:f}. For the second part of the proof, we prove by induction. It can be easily verified that for any $i, j$ such that $j - i = 1$, we have $f(i, j) = 0.4\sigma_{ij} \le 1$. Now, suppose that the inequality holds for any $i, j$ with $j - i = l - 1$, then for any $j', i'$ with $j' - i' = l$, we have
  \begin{align*}
      f(i, j) \le 0.4 \sigma_{ij} + \sum_{m = i+1}^{j-1}0.4 \sigma_{im} = 0.4 \sum_{m= i+1}^{j}\sigma_{im} \le 0.8
  \end{align*}
  The last inequality follows from Lemma \ref{lemma:bound_sigma}.
\end{proof}

\subsection{{Proof of Lemma \ref{lemma:suboptimal_set_upper_bound}}}\label{proof:suboptimal_set_upper_bound}
\begin{proof}
  Recall that $P_{1}^{*} = \PP(a_{1}|s^{*}), \cdots, P_{k}^{*} = \PP(a_{k}| s^{*})$. 
  Define $\delta_{ij} = P_{i}^{*} - P_{j}^{*}$, $\Delta_{l} = \sum_{i = l}^{k}\delta_{li}$. Define $t_{l}$ to be the last time step with $\rho(t_{l}) \ge P_{l}^{*}$. Denote $t'_{l}$ to be the number of times $s(t) = s^{*}$ for $t \le t_{l}$.

  \textbf{Case I: $\Delta_{l} \ge \frac{\epsilon}{10}$.}

  By Lemma \ref{lemma:strong-sum}, we have
  \begin{align*}
    \sqrt{4\alpha kn \ln T \rbr{t_{l} - \frac{l}{k}t_{l}'}} \ge \sum_{i=1}^{l}P_{i}^{*}t_{l} + (k-l)P_{l}^{*}t_{l} - \sum_{i=1}^{n}\sum_{c=1}^{t_{l}}P_{i}(c) - \sum_{i=1}^{l-1}\delta_{il}(t_{i} - t_{i}') - nP_{l}^{*}
  .\end{align*}
  Note that
  \begin{align*}
    \sum_{i=1}^{l}P_{i}^{*}t_{l} + (k-l)P_{l}^{*} - \sum_{i=1}^{k}P_{i}^{*}  = \sum_{i = l}^{k}\delta_{li} = \Delta_{l}
  .\end{align*}
  By the fact $\sum_{i=1}^{k}P_{i}(t) \le \sum_{i=1}^{k}P_{i}^{*} - \epsilon$ for suboptimal assortmet, we have 
  \begin{align*}
    \sqrt{4\alpha kn \ln T \rbr{t_{l} - \frac{l}{k}t_{l}'}} &\ge \Delta_{l}t_{l} + \epsilon\rbr{t_{l} - t'_{l}} - \sum_{i=1}^{l - 1}\delta_{il}\rbr{t_{i} - t'_{i}} - nP_{l}^{*} \\
     & \ge \rbr{\Delta_{l} + \epsilon}\rbr{t_{l} - \frac{l}{k}t'_{l}} - \frac{\epsilon (k-l) - \Delta_{l}l }{k}t'_{l} - \sum_{i=1}^{l-1}\delta_{il}(t_{i}-t'_{i}) - nP_{l}^{*}
  .\end{align*}
  For $\ln T \ge 5$, with the fact $k \ge \Delta_{l}, 1 \ge P_{l}^{*}, \alpha \ge 1$, we have $4n(\Delta_{l}+\epsilon)P_{l}^{*} \le 0.8\alpha k n \ln T$, we therefore have the following bound for $\sqrt{t_{l} - \frac{l}{k}t'_{l}}$, 
  \begin{align}\label{eq:1}
    \sqrt{t_{l} - \frac{l}{k}t'_{l}}\le \frac{\sqrt{4\alpha k n \ln T}\rbr{1 + \sqrt{1.2 + \frac{\sum_{i=1}^{l-1}4\rbr{\Delta_{l} + \epsilon}\delta_{il}}{4\alpha kn \ln T}\rbr{t_{i} - t'_{i}} + \frac{4\rbr{\Delta_{l} + \epsilon}\frac{\epsilon(k-l) - \Delta_{l}l}{k}}{4\alpha kn \ln T}t'_{l}}}}{2\rbr{\Delta_{l} + \epsilon}}
  .\end{align}
  For simplicity, we write $t_{l}  - t'_{l}$ in the following form
  \begin{align*}
    t_{l} - t'_{l} = c_{l}\frac{4\alpha k n \ln T}{\rbr{\Delta_{l} + \epsilon}^{2} }
  .\end{align*}
  \Cref{eq:1} can be then rewriten as 
  \begin{align*}
    \sqrt{t_{l} - \frac{l}{k}t'_{l}} \le \frac{1}{2}\rbr{1 + \sqrt{1.2 + \sum_{i=1}^{l - 1}\frac{4\delta_{il}\rbr{\Delta_{l} + \epsilon}}{\rbr{\Delta_{i} + \epsilon}^{2}}c_{i} + \frac{4\rbr{\Delta_{l} + \epsilon}\frac{\epsilon(k-l) - \Delta_{l}l}{k}}{4\alpha kn \ln T}t'_{l}}}\frac{\sqrt{4\alpha kn \ln T}}{\Delta_{l} + \epsilon}
  .\end{align*}
  Furhter, define $\sigma_{i, l} = \frac{4\delta_{il}\rbr{\Delta_{l} + \epsilon}}{\rbr{\Delta_{i} + \epsilon}^{2}}$, we have
  \begin{align*}
    \sqrt{t_{l} - \frac{l}{k}t'_{l}} \le \frac{1}{2}\rbr{1 + \sqrt{1.2 + \sum_{i = 1}^{l - 1}\sigma_{il}c_{i} + \frac{4\rbr{\Delta_{l} + \epsilon}\frac{\epsilon(k-l) - \Delta_{l}l}{k}}{4\alpha kn \ln T}t'_{l}}}\frac{\sqrt{4\alpha kn \ln T}}{\Delta_{l}+\epsilon}
  .\end{align*}
  By the fact $(1+a)^{2} \le 1.1 a^{2} + 11$ for any real number $a$, we have
  \begin{align*}
    t_{l} - \frac{l}{k}t'_{l} \le \frac{1}{4}\rbr{11 + 1.32 + 1.1 \sum_{i=1}^{l - 1}\sigma_{il}c_{i}}\frac{4\alpha kn \ln T}{\rbr{\Delta_{l} + \epsilon}^{2}} + 1.1\frac{\epsilon(k-l) - \Delta_{l}l}{k\rbr{\Delta_{l} + \epsilon}}t'_{l}   
  .\end{align*}
    Since $\Delta_{l} \ge \frac{\epsilon}{10}$, we have $\frac{\epsilon}{3}\rbr{k - l} \le \frac{2}{3} \Delta_{l}\rbr{k +\frac{1}{2}l}$, which imples $\epsilon(k -l) - \Delta_{l}l \le \frac{2}{3} (k-l)(\Delta_{l}+\epsilon)$. Therefore we have
  \begin{align*}
    & t_{l} - \frac{l}{k}t'_{l} = \rbr{3.08 + 0.275\sum_{i= 1}^{l - 1}\sigma_{il}c_{i} }\frac{4\alpha kn \ln T}{\rbr{\Delta_{l}+\epsilon}^{2}} + \frac{k-l}{k}t'_{l} \\
    & \implies t_{l} - t'_{l} \le \rbr{3.08 + 0.275\sum_{i= 1}^{l - 1}\sigma_{il}c_{i} }\frac{4\alpha kn \ln T}{\rbr{\Delta_{l}+\epsilon}^{2}}
  .\end{align*}
  Plug in the convention of $c_{l}$, we have
  \begin{align*}
    c_{l} \le 3.08 + 0.275\sum_{i= 1}^{l - 1}\sigma_{il}c_{i}
  .\end{align*}
   With Lemma \ref{lemma:f}, We can use $f(i,j)$ to upper bound $c_{l}$. First define $c'_{l} = \frac{c_{l}}{10}$, which implies that
  \begin{align*}
    c'_{l} \le 0.308 + 0.275 \sum_{i=1}^{l-1}\sigma_{il}c'_{i}
  .\end{align*}
  Next we proceed to show that 
  \begin{align}\label{eq:2}
    c'_{l} & \le 0.308 + \sum_{i = 1}^{l-1}f(i, l)
  .\end{align}
  We prove \Cref{eq:2} by induction. For $l = 1, 2$, we have
  \begin{align*}
    c'_{1} \le 0.308, \quad c'_{2} \le 0.308 + 0.275 \sigma_{12} c'_{1} \le 0.308 + 0.275\sigma_{12} = 0.308 + f(1, 2)
  .\end{align*}
  Suppose \Cref{eq:2} holds for $c'_{l-1}$, then we have
  \begin{align*}
    c'_{l} & \le 0.308 + 0.275\sum_{i=1}^{l - 1}\sigma_{il}c_{i} \le 0.308 + 0.275\sum_{i=1}^{l-1}\sbr{0.308 + \sum_{j=1}^{i-1}f(j, i)}\sigma_{il} \\
          & \le 0.308 + \sum_{i=1}^{l -1}\sbr{0.275\sigma_{il} + 0.275\sum_{j=1}^{i-1}f(j, i)\sigma_{il}} \\
          & \le 0.308 + \sum_{j=1}^{l-1}\sum_{i=j+1}^{l-1}0.275f(j, i)\sigma_{il} + \sum_{i=1}^{l-1}0.275\sigma_{il} \\
          & \le 0.308 + \sum_{j=1}^{l-1}\sbr{0.275\sigma_{jl} + \sum_{i=j+1}^{l-1}0.275f(j, i)\sigma_{il}} \\
          & \le 0.308 + \sum_{i=1}^{l - 1}f(i, l)
  .\end{align*}
  The last inequality follows from the first equation in Lemma \ref{lemma:f}. Combining with the second inequality in Lemma \ref{lemma:f}, we have $c_l \le 10 l$. This completes the proof of the first case in Lemma \ref{lemma:suboptimal_set_upper_bound}.

  \textbf{Case II: $\Delta_{l} < \frac{\epsilon}{10}$.}

  Denote $l'$ to be the lagest $i$ with $\Delta_{i} \ge \epsilon / 10$. By definition, we know $l > l'$. Applying \Cref{lemma:strong-sum} to all arms, we have that
  \begin{align*}
    2\sqrt{\alpha k n \rbr{t_{l} - t'_{l}} \ln T} &\ge \sum_{i=k+1}^{n}P_{l}^{*}N_{i}(t_{l}) + \sum_{i=1}^{k}P_{i}^{*}N_{i}(t_{l}) - \sum_{i=1}^{n}\sum_{c=1}^{t_{l}}P_{i}(c) - nP_{l}^{*} \\
                                                      &\ge \epsilon\rbr{t_{l} - t'_{l}} - \sum_{i=1}^{l-1}\delta_{ik}(t_{i} - t'_{i}) - nP_{l}^{*}
  .\end{align*}
  
  Solving for $t_{l} - t'_{l}$, we have
  \begin{align*}
    \sqrt{t_{l} - t'_{l}} \le \frac{\sqrt{4\alpha kn \ln T}+ \sqrt{4.8\alpha kn \ln T + \sum_{i=1}^{l-1} 4\delta_{il}\epsilon \rbr{t_i - t'_{i}}}}{\epsilon}
  .\end{align*}
  Similar as $c_{l}$, we write $t_{l}$ for $l > l'$ as
  \begin{align*}
    t_{l} - t'_{l} = d_{n}\frac{4\alpha kn \ln T}{\epsilon^{2}}
  .\end{align*}
  Therefore
  \begin{align*}
    \sqrt{t_{l} - t'_{l}} &\le \frac{\sqrt{4\alpha kn \ln T} \rbr{1 + \sqrt{1.2 + \sum_{i=1}^{l'}\frac{4\delta_{il}\epsilon }{\rbr{\Delta_{i} + \epsilon}^{2}}c_{i} + \sum_{i=l'+1}^{l-1}\frac{4\delta_{il}\epsilon}{\epsilon^{2}}d_{i}}} }{\epsilon} \\
                  &\le \frac{\sqrt{4\alpha kn \ln T} \rbr{1 + \sqrt{1.2 + \sum_{i=1}^{l'}\frac{4\delta_{il}\rbr{\Delta_{l} + \epsilon} }{\rbr{\Delta_{i} + \epsilon}^{2}}c_{i} + 1.21\sum_{i=l'+1}^{l-1}\frac{4\delta_{il}\rbr{\Delta_{i} + \epsilon}}{\rbr{\Delta_{i}+ \epsilon}^{2}}d_{i}}} }{\epsilon} 
  .\end{align*}
  The second inequality follows from $\frac{\rbr{\Delta_{i} + \epsilon}^{2}}{\epsilon^{2}} \le 1.21$ as $\Delta_{i} < \frac{\epsilon}{10}$ for all $i > l'$. Simplify the inequality, we have 
  \begin{align*}
    \sqrt{d_{l}} \le \frac{1}{4}\rbr{1 + \sqrt{1.2 + \sum_{i=1}^{l'}\sigma_{il}c_{i} + 1.21\sum_{i=l'+1}^{l-1}\sigma_{il}d_{i} }}
  .\end{align*}
  Again use the fact that $(1+a)^{2} \le 11 + 1.1 a^{2}$, we have
  \begin{align*}
    d_{l} &\le \frac{1}{4}\rbr{11 + 1.32 + 1.1\sum_{i=1}^{l'}\sigma_{il}c_{i} + 1.331\sum_{i=l'+1}^{l-1}\sigma_{il}d_{i}} \\
          &\le 3.08 + 0.275\sum_{i=1}^{l'-1}\sigma_{il}c_{i} + 0.34 \sum_{i=l'+1}^{l-1}\sigma_{i}d_{i}
  .\end{align*}
  Recall that we've defined $f(i, j) = 0.4\sigma_{ij} + \sum_{k = i+1}^{j-1}0.4\sigma_{ik}f(k, j)$. Similar to showing $c_l \le 3.08 + \sum_{i=1}^{l-1}f(i, l)$, we can define $d_l' = d_l / 10$ and have
  \begin{align*}
    d'_{l} & \le 0.308 + 0.275\sum_{i=1}^{l'}\sigma_{il}c'_{i} + 0.34 \sum_{i=l'+1}^{l-1}\sigma_{il}d_i \\
          & \le 0.308 + 0.275\sum_{i=1}^{l'}\sbr{0.308 + \sum_{j=1}^{i-1}f(j, i)}\sigma_{il} + 0.34 \sum_{i=l'+1}^{l-1}\sbr{0.308 + \sum_{j=1}^{i-1}f(j, i)}\sigma_{il}\\
          & \le 0.308 + \sum_{i=1}^{l -1}\sbr{0.34\sigma_{il} + 0.34\sum_{j=1}^{i-1}f(j, i)\sigma_{il}} \\
          & \le 0.308 + \sum_{j=1}^{l-1}\sum_{i=j+1}^{l-1}0.34f(j, i)\sigma_{il} + \sum_{i=1}^{l-1}0.34\sigma_{il} \\
          & \le 0.308 + \sum_{j=1}^{l-1}\sbr{0.34\sigma_{jl} + \sum_{i=j+1}^{l-1}0.34f(j, i)\sigma_{il}} \\
          & \le 0.308 + \sum_{i=1}^{l - 1}f(i, l)
  .\end{align*}
  
  Therefore we have $d_{l} \le 3.08 + 10\sum_{i=1}^{l-1}f(i,l) \le 10 l$, which completes the proof for the second case.
\end{proof}

\subsection{{Proof of \Cref{lemma:regret_decomposition}}}

\begin{proof}
  Note that by \Cref{assumption:weak}, we have $\rho(T) \ge P(a_k|s^*)$ which implies $R(T) \le R(t_k)$. Plug in Lemma \ref{lemma:strong-sum} with $l = k$, we have
  \begin{align*}
        \sqrt{4\alpha kn \ln T \rbr{t_{k} - t_{k}'}} \ge \sum_{i=1}^{k}P(a_i|s^*)t_{k} - \sum_{i=1}^{n}\sum_{c=1}^{t_{k}}P_{i}(c) - \sum_{i=1}^{k-1}\delta_{ik}(t_{i} - t_{i}') - nP(a_k|s^*).
  \end{align*}
  Note that $R(t_k) = \sum_{i=1}^{k}P(a_i|s^*)t_{k} - \sum_{i=1}^{n}\sum_{c=1}^{t_{k}}P_{i}(c)$, Rearranging the terms, we have
  \begin{align*}
      R(T) \le R(t_k) \le \sqrt{4\alpha kn \ln T \rbr{t_{k} - t_{k}'}} + \sum_{l=1}^{k-1}\delta_{lk}\rbr{t_{l} - t'_{l}} + nP(a_{k}|s^{*})
  \end{align*}
\end{proof}

\section{Proof for \Cref{sec:regret_lower_bound}}\label{proof:regret_lower_bound}

\subsection{Regret Lower bound for Feedback Model $\Mcal1$}
We prove the lower bound for the feedback model $\Mcal1$ with mutually exclusive rewards. By constructing a family of environments $\Ecal_i, i\in[n]$. We define the arm set as $\Acal = \cbr{a_1, \cdots, a_{n+k-1}}$.

In environment $\Ecal_i$, the optimal set is $\cbr{a_i, a_{n+1},a_{n+2}\cdots,a_{n+k-1}}$. We assume those arms to have $\frac{1}{k+1}$ probability of receiving positive reward in any set. All other arms not belonging to the optimal set have $\frac{1}{k+1}-{\epsilon}$ probability of receiving positive reward in any set. It's easy to verify that all environments $\Ecal_i$ satisfies \Cref{assumption:weak} and the minimum gap between optimal and sub-optimal set is $\epsilon$. We then have the following regret lower bound.

Denote $q_i$ to be the distribution of $T$-step history induced by $\Ecal_i$. We then have the following Lemma:
\begin{lemma}[Lower Bound for Each Arm] Under feedback model $\Mcal1$, let $\phi$ be an algorithm for the combinatorial bandits problem with \Cref{assumption:weak}, such that the regret is $R_\phi(T)=o(T^a)$ for all $a>0$. Then for the environment $\Ecal_1$ we have $\EE_{q_1}(N_j(T)) = \Omega\rbr{\frac{\log T}{k\epsilon^2}}$ for all arm $a_j$.
\end{lemma}
\begin{proof}
	For a fixed $j\notin\cbr{1, n}$, we define the event $B_j = \cbr{N_j(T)\le \log T/\epsilon ^2}$. If $q_1(B_j) < 1/3$, we have
	\begin{align*}
		\EE_{q_1}(N_j(T))\ge q_1(B_j^c)\log T/\epsilon ^2=\Omega(\log T/\epsilon ^2)
	\end{align*}
	Now suppose $q_1(B_j)\ge 1/3$. Note that in environment $\Ecal_j$, the algorithm will incur at least $\epsilon $ regret if not selecting $a_j$, Therefore we have $\EE_{q_j}(T-N_j(T)) = o(T^a)$. By Markov's inequality, we have
	\begin{align*}
		q_j(B_j) = q_j\rbr{\cbr{T-N_j(T)>T-\log T / \epsilon ^2}} \le \frac{\EE_{q_j}(T-N_j(T))}{T-\log T / \epsilon ^2} = o(T^{a-1})
	\end{align*}
	From \citep{karp2007noisy}, we know that for any event $B$ and two distributions $p,q$ with $p(B)>1/3$ and $q(B)<1/3$, we have
	\begin{align*}
		D_{\text{KL}}(p;q)\ge \frac{1}{3}\ln (\frac{1}{3q(B)}) - \frac{1}{e}
	\end{align*}
	Putting $q_1, q_j$ and $B_j$ into the inequality above, we have
	\begin{align*}
		D_{\text{KL}}(q_1;q_j) \ge \frac{1}{3}\ln (\frac{1}{3o(T^{a-1})})-\frac{1}{e}=\Omega(\ln T)
	\end{align*}
	On the other hand, since the only different arm between $\Ecal_1$ and $\Ecal_j$ is arm $a_j$. We need to bound the KL-divergence by playing any  set containing $a_j$. Suppose $p$ is a categorical distribution with parameters $p_1, ..., p_k$ for $k$ items and $p'$ is another categorical distribution with parameters $p_1 - \epsilon_1, ..., p_k - \epsilon_k$. Then we have
	\begin{align*}
	    D_{\text{KL}}(p, p') = \sum_{i = 1}^k(p'_i + \epsilon_i)\log \frac{p'_i + \epsilon_i}{p'_i} \le \sum_{i=1}^k(p'_i + \epsilon_i)\frac{\epsilon_i}{p'_i} = \sum_{i=1}^k\frac{\epsilon_i^2}{p'_i},
	\end{align*}
	where the last inequality holds because $\sum_{i=1}^k\epsilon_i = 0$. Therefore we can directly bound the KL-divergence of $q_1$ and $q_j$ by
	\begin{align*}
	    D_\text{KL}(q_1;q_j) \le C\EE(N_j(T))k\epsilon ^2,
	\end{align*}
	where $C$ is a problem-independent constant. It then directly implies that
	\begin{align*}
		C\EE(N_j(T))k\epsilon ^2 = \Omega(\log T) \implies \EE_{q_1}(N_j(T)) = \Omega\rbr{\frac{\log T}{k\epsilon^2}}
	\end{align*}
	which completes the proof.
\end{proof}

From \Cref{lemma:regret_lower_bound}, we know that in $\Ecal_1$ each arm will be played for $\Omega(\log T/k\epsilon^2)$, and each time a sub-optimal arm is played, it induces at least $\epsilon$ regret. Since we have $n+k-1$ arms in $\Acal$, it immediately implies that the regret is lower bounded by $\Omega(n\log T/k\epsilon)$. For the algorithm that doesn't satisfy the assumption in \Cref{lemma:regret_lower_bound} (i.e. for some $a>0$, the $o(T^a)$ regret bound doesn't hold), the lower bound holds directly. As a summary, we have \Cref{thm:regret_lower_bound}.

\subsection{Regret Lower bound for Feedback Model $\Mcal2$}
The environment construction is similar to the one for $\Mcal1$. The only difference is to replace all $\frac{1}{k+1}$ with $\frac{1}{2}$. Accordingly, we have

\begin{lemma}[Lower Bound for Each Arm]\label[lemma]{lemma:regret_lower_bound} Under feedback model $\Mcal2$, let $\phi$ be an algorithm for the combinatorial bandits problem with \Cref{assumption:weak}, such that the regret is $R_\phi(T)=o(T^a)$ for all $a>0$. Then for the environment $\Ecal_1$ we have $\EE_{q_1}(N_j(T)) = \Omega\rbr{\frac{\log T}{\epsilon^2}}$ for all arm $a_j$.
\end{lemma}
Similar to previous subsection, it implies a $\Omega(n\log T/\epsilon)$ lower bound.

\section{Experiment Setup}\label{APDX:exp_setup}
\subsection{Multinomial Logit}
In this environment, the reward is generated according to a multinomial logit model
\begin{align*}
    \PP(a_i | s(t)) = \frac{v_i}{1 + \sum_{a_i\in s(t)}v_i},\quad \PP(a_0 | s(t)) = \frac{1}{1 + \sum_{a_i\in s(t)}v_i}
\end{align*}
where $v_i$ is the value associated with each arm $a_i$, determining the reward probability. In this experiment, we set $v_i = 1 - 0.04i$ with $i\in[20]$. The size of set is set to $k=10$, and the optimal set is $s^*$ is composed by arms from $a_1$ to $a_{10}$. The regret of set $s(t)$ is given by
\begin{align*}
    reg(s(t)) = \frac{1}{1 + \sum_{a_i\in s(t)}v_i} - \frac{1}{1 + \sum_{a_i \in s^*}v_i}
\end{align*}

\subsection{Random Utility Model}
In this environment, for an set $s(t)$ at time step $t$, each arm $a_i\in s(t)$ will independently draw a Gaussian distributed random variable $x_i \sim \Ncal(\mu_i, 1)$, where $\mu_i$ is the mean associated with each arm $a_i$. Along with that $a_0$ will draw a $x_0 \sim \Ncal(2, 1)$. The arm $a_i$ (including $a_0$) with highest $x_i$ will receive reward. Thus we have the probability of $a_i$ getting reward as
\begin{align*}
    \PP(a_i|s(t)) = \PP(x_i = \max_{a_j\in s(t) \cup \cbr{a_0}}x_j)
\end{align*}
Here, we set $\mu_i = 1 - 0.04i$ with $i\in[20]$. The size of set is set to $k=5$, and the optimal set $s^*$ is composed by the arms from $a_1$ to $a_5$. For the convenience of computation, the regret of set $s(t)$ is defined slightly different as
\begin{align*}
    reg(s(t)) = \sum_{a_i \in s^*}\mu_i - \sum_{a_i \in s(t)}\mu_i
\end{align*}
Once $s(t)$ recovers the optimal set $s^*$, which maximizes the probability of $s(t)$ receiving reward, we will have this regret $reg(s(t)) = 0$.

\subsection{Preference Matrix}
In this environment, the probability of one arm $a_i$ getting reward is fully specified by a preference matrix. For ease of representation, we set the number of arms to $n=10$ and the size of set to $k=2$. Th total number of sets is 45, much lesser than the previous two environments. However, with a specially designed preference matrix (including the loop in preference, etc), the environment turns out to be the hardest.

We set $M$ to be the preference matrix with $M_{i,j} = \PP(a_i|s(t) = \cbr{a_i, a_j}) - \PP(a_j|s(t) = \cbr{a_i, a_j})$. We set the optimal set to be $s^* = \cbr{a_1, a_2}$ with $\PP(a_0|s^*) = 0.08$. For all other sets $s$ which are sub-optimal, we set $\PP(a_0|s) = 0.1$. The preference matrix $M$ is given in \Cref{table:preference_matrix}.

\begin{table}[h]
\centering
\begin{tabular}{|c|c|c|c|c|c|c|c|c|c|c|}
\hline
      & $a_1$  & $a_2$  & $a_3$  & $a_4$ & $a_5$ & $a_6$ & $a_7$ & $a_8$ & $a_9$ & $a_{10}$ \\ \hline
$a_1$  & --    & 0.02  & 0.05  & 0.1  & 0.1  & 0.2  & 0.25 & 0.3  & 0.3  & 0.3   \\ \hline
$a_2$  & -0.02 & --    & 0.05  & 0.1  & 0.1  & 0.2  & 0.25 & 0.3  & 0.3  & 0.3   \\ \hline
$a_3$  & -0.05 & -0.05 & --    & 0.45 & 0.45 & 0.45 & 0.45 & 0.45 & 0.45 & 0.45  \\ \hline
$a_4$  & -0.1  & -0.1  & -0.45 & --   & -0.3 & 0.3  & 0    & 0    & 0    & 0     \\ \hline
$a_5$  & -0.1  & -0.1  & -0.45 & 0.3  & --   & -0.3 & 0    & 0    & 0    & 0     \\ \hline
$a_6$  & -0.2  & -0.2  & -0.45 & -0.3 & 0.3  & --   & 0    & 0    & 0    & 0     \\ \hline
$a_7$  & -0.25 & -0.25 & -0.45 & 0    & 0    & 0    & --   & 0    & 0    & 0     \\ \hline
$a_8$  & -0.3  & -0.3  & -0.45 & 0    & 0    & 0    & 0    & --   & 0    & 0     \\ \hline
$a_9$  & -0.3  & -0.3  & -0.45 & 0    & 0    & 0    & 0    & 0    & --   & 0     \\ \hline
$a_{10}$ & -0.3  & -0.3  & -0.45 & 0    & 0    & 0    & 0    & 0    & 0    & --    \\ \hline
\end{tabular}
\caption{Preference Matrix $M$}\label{table:preference_matrix}
\end{table}

We can see that when $a_3$ pairs with any other sub-optimal arm, it will have a higher chance of getting reward than $a_1$ and $a_2$. It makes $a_3$ the seemingly best single arm. Also note that when $a_4$ pairs with $a_5$, $a_5$ will have a higher chance of getting reward. Similarly, $a_6$ will win over $a_5$ and $a_4$ will win over $a_6$. The preference therefore forms a loop among $a_4, a_5, a_6$.

The regret of $s(t)$ is given by
\begin{align*}
    reg(s(t)) = \PP(a_0|s(t)) - \PP(a_0|s^*)
\end{align*}

\subsection{Random Weak Optimal Set Consistency}
In this environment, we randomly generate the environment with Algorithm \ref{alg:env_generation} that satisfies the Assumption \ref{assumption:weak}. 
\begin{algorithm}[h]
\centering 
\begin{algorithmic}[1]
\STATE \textbf{Input:} Number of Arms $n$, set Size $k$.
\STATE Set set $s^* = \{1, 2, \cdots, k\}$ be the optimal set. Randomly Sample $\mathbb{P}(a|s^*)\sim \text{Uniform}(0, \frac{1}{k})$.
\FOR{set $s\neq s^*$}
\WHILE{$\sum_{a\in s} P(a|s) > \sum_{a^*\in s^*} \mathbb{P}(a^*|s^*)$}
\FOR{$a\in s$}
\IF{$a\in s^*$}
\STATE Sample $\mathbb{P}(a|s) \sim\text{Uniform}(\mathbb{P}(a|s^*), \frac{1}{k})$.
\ELSE
\STATE Sample $\mathbb{P}(a|s) \sim\text{Uniform}(0, \frac{1}{k})$.
\ENDIF
\ENDFOR
\ENDWHILE
\ENDFOR 
\end{algorithmic}
\caption{\textsc{Generating Environment Satisfies Assumption \ref{assumption:weak}.} }
\label{alg:env_generation}
\end{algorithm}

By construction, the environment satisfies Assumption \ref{assumption:weak}. Moreover, as we randomly sample the feedback for each set randomly, it's not necessary for the generated environment to satisfy more stronger Assumption, e.g. the strict preference order. The regret of set $s(t)$ is given by
\begin{align*}
    reg(s(t)) = \sum_{a\in s(t)} \mathbb{P}(a|s_t) - \sum_{a^*\in s^*}\mathbb{P}(a^*|s^*).
\end{align*}

\end{document}